\documentclass[a4paper]{article}

\usepackage[backend=bibtex, style=ieee]{biblatex}
\usepackage{parskip}
\usepackage{nicefrac}
\usepackage{mathtools}
\usepackage{amsthm}
\usepackage{paralist}
\usepackage[bookmarks=true, colorlinks=true, linkcolor=blue!50!black, citecolor=orange!50!black, urlcolor=orange!50!black, pdfencoding=unicode]{hyperref}
\usepackage[capitalise]{cleveref}

\theoremstyle{definition}

\usepackage[a4paper, lmargin=0.1666\paperwidth, rmargin=0.1666\paperwidth, tmargin=0.1111\paperheight, bmargin=0.1111\paperheight]{geometry} 
\usepackage[framemethod=tikz,xcolor=true]{mdframed}

\usepackage{comment}
\usepackage{siunitx}
\usepackage{relsize}
\usepackage{ifthen}
\usepackage[colorinlistoftodos]{todonotes}

\usepackage[caption=false]{subfig}

\usepackage[vlined,ruled,linesnumbered]{algorithm2e}
\usepackage{graphics} 
\usepackage{rotating}
\usepackage{color}
\usepackage{enumerate}
\usepackage[T1]{fontenc}
\usepackage{psfrag}
\usepackage{epsfig} 
\usepackage{booktabs}
\usepackage{graphicx,url}
\usepackage{multirow}
\usepackage{array}
\usepackage{latexsym}
\usepackage{amsfonts}
\usepackage{amsmath}
\usepackage{amssymb}
\usepackage{xstring}
\usepackage[noend]{algorithmic}
\usepackage{multirow}
\usepackage{xcolor}
\usepackage{prettyref}
\usepackage{flexisym}
\usepackage{bigdelim}
\usepackage{breqn} 
\usepackage{listings}

\usepackage{enumitem}
\usepackage{xspace}
\usepackage{bm}
\graphicspath{{./figures/}}
\usepackage{tikz}
\usetikzlibrary{matrix,calc}

\usepackage{mdwlist}

\makecompactlist{itemize}{stditemize}

\newrefformat{prob}{Problem\,\ref{#1}}
\newrefformat{def}{Definition\,\ref{#1}}
\newrefformat{sec}{Section\,\ref{#1}}
\newrefformat{sub}{Section\,\ref{#1}}
\newrefformat{prop}{Proposition\,\ref{#1}}
\newrefformat{app}{Appendix\,\ref{#1}}
\newrefformat{alg}{Algorithm\,\ref{#1}}
\newrefformat{cor}{Corollary\,\ref{#1}}
\newrefformat{thm}{Theorem\,\ref{#1}}
\newrefformat{lem}{Lemma\,\ref{#1}}
\newrefformat{fig}{Fig.\,\ref{#1}}
\newrefformat{tab}{Table\,\ref{#1}}

\newtheorem{theorem}{Theorem}

\newtheorem{corollary}[theorem]{Corollary}

\newtheorem{definition}[theorem]{Definition}
\newtheorem{proposition}[theorem]{Proposition}

\newcommand{\cf}{\emph{cf.}\xspace}

\newcommand{\bdmath}{\begin{dmath}}
\newcommand{\edmath}{\end{dmath}}
\newcommand{\beq}{\begin{equation}}
\newcommand{\eeq}{\end{equation}}
\newcommand{\bdm}{\begin{displaymath}}
\newcommand{\edm}{\end{displaymath}}
\newcommand{\bea}{\begin{eqnarray}}
\newcommand{\eea}{\end{eqnarray}}
\newcommand{\beal}{\beq \begin{array}{ll}}
\newcommand{\eeal}{\end{array} \eeq}
\newcommand{\beas}{\begin{eqnarray*}}
\newcommand{\eeas}{\end{eqnarray*}}
\newcommand{\ba}{\begin{array}}
\newcommand{\ea}{\end{array}}
\newcommand{\bit}{\begin{itemize}}
\newcommand{\eit}{\end{itemize}}
\newcommand{\ben}{\begin{enumerate}}
\newcommand{\een}{\end{enumerate}}

\newcommand{\setE}{\textsf{E}}

\newcommand{\setU}{\textsf{U}}

\newcommand{\setal}{~\emph{et~al.}\xspace}
\newcommand{\eg}{\emph{e.g.,}\xspace}
\newcommand{\ie}{\emph{i.e.,}\xspace}
\newcommand{\myParagraph}[1]{{\bf #1.}\xspace}

\newcommand{\hide}[1]{}

\newcommand{\hiddenText}{{\color{gray} hidden text.}}
\newcommand{\hideWithText}[1]{\hiddenText}

\newcommand{\Real}[1]{ { {\mathbb R}^{#1} } }

\newcommand{\SEthree}{\ensuremath{\mathrm{SE}(3)}\xspace}

\newcommand{\blue}[1]{{\color{blue}#1}}

\newcommand{\linkToPdf}[1]{\href{#1}{\blue{(pdf)}}}
\newcommand{\linkToPpt}[1]{\href{#1}{\blue{(ppt)}}}
\newcommand{\linkToCode}[1]{\href{#1}{\blue{(code)}}}
\newcommand{\linkToWeb}[1]{\href{#1}{\blue{(web)}}}
\newcommand{\linkToVideo}[1]{\href{#1}{\blue{(video)}}}
\newcommand{\linkToMedia}[1]{\href{#1}{\blue{(media)}}}
\newcommand{\award}[1]{\xspace} 

\DeclareMathOperator{\Type}{Type}
\newcommand{\setVars}{V}

\newcommand{\modules}{modules\xspace}
\newcommand{\module}{module\xspace}

\newcommand{\unit}{processor\xspace}
\newcommand{\units}{processors\xspace}

\newcommand{\uniti}{i\xspace}
\newcommand{\unitj}{j\xspace}

\newcommand{\toAdd}[1]{}

\tikzset{
graphnode/.style={
  circle,
  inner sep=0pt,
  text width=4mm,
  align=center,
  draw=black,
  fill=white
  }
}
\newcommand*\circled[1]{\tikz[baseline=(char.base)]{
            \node[graphnode] (char) {#1};}}
\newcommand*\node[1]{\circled{#1}}

\DeclarePairedDelimiter{\ceil}{\lceil}{\rceil}
\DeclarePairedDelimiter{\floor}{\lfloor}{\rfloor}

\usepackage[prefix=tikzsym]{tikzsymbols}

\DeclarePairedDelimiter{\pair}{\langle}{\rangle}

\newcommand{\cat}[1]{\mathcal{#1}}
\newcommand{\Cat}[1]{\mathbf{#1}}

\newcommand{\const}[1]{\mathtt{#1}}

\newcommand{\prop}{\const{Prop}}

\newcommand{\tn}{\textnormal}

\renewcommand{\ss}{\subseteq}
\newcommand{\rr}{\mathbb{R}}
\newcommand{\smset}{\Cat{Set}}

\newcommand{\ol}[1]{\overline{#1}}

\newcommand{\nn}{\mathbb{N}}

\newcommand{\true}{\const{true}}
\newcommand{\false}{\const{false}}
\newcommand{\imp}{\Rightarrow}

\newcommand{\rest}[3]{#1\big|_{[#2,#3]}}

\usetikzlibrary{
	cd,
	math,
	decorations.markings,
	decorations.pathreplacing,
	decorations.pathmorphing,
	positioning,
	arrows.meta,
	circuits.logic.US,
	shapes,
	calc,
	fit,
	quotes}

\tikzset{
     oriented WD/.style={
        every to/.style={out=0,in=180,draw},
        label/.style={
           font=\everymath\expandafter{\the\everymath\scriptstyle},
           inner sep=0pt,
           node distance=2pt and -2pt},
        semithick,
        node distance=1 and 1,
        decoration={markings, mark=at position \stringdecpos with \stringdec},
        ar/.style={postaction={decorate}},
        execute at begin picture={\tikzset{
           x=\bbx, y=\bby,
           }}
        },
     string decoration/.store in=\stringdec,
     string decoration={\arrow{stealth};},
     string decoration pos/.store in=\stringdecpos,
     string decoration pos=.7,
     bbx/.store in=\bbx,
     bbx = 1.5cm,
     bby/.store in=\bby,
     bby = 1.5ex,
     bb port sep/.store in=\bbportsep,
     bb port sep=1.5,
     bb port length/.store in=\bbportlen,
     bb port length=4pt,
     bb penetrate/.store in=\bbpenetrate,
     bb penetrate=0,
     bb min width/.store in=\bbminwidth,
     bb min width=1cm,
     bb rounded corners/.store in=\bbcorners,
     bb rounded corners=2pt,
     bb small/.style={bb port sep=1, bb port length=2.5pt, bbx=.4cm, bb min width=.4cm, 
bby=.7ex},
		 bb medium/.style={bb port sep=1, bb port length=2.5pt, bbx=.4cm, bb min width=.4cm, 
bby=.9ex},
     bb/.code 2 args={
        \pgfmathsetlengthmacro{\bbheight}{\bbportsep * (max(#1,#2)+1) * \bby}
        \pgfkeysalso{draw,minimum height=\bbheight,minimum width=\bbminwidth,outer 
sep=0pt,
           rounded corners=\bbcorners,thick,
           prefix after command={\pgfextra{\let\fixname\tikzlastnode}},
           append after command={\pgfextra{\draw
              \ifnum #1=0{} \else foreach \i in {1,...,#1} {
                 ($(\fixname.north west)!{\i/(#1+1)}!(\fixname.south west)$) +(-
\bbportlen,0) 
  coordinate (\fixname_in\i) -- +(\bbpenetrate,0) coordinate (\fixname_in\i')}\fi 
              \ifnum #2=0{} \else foreach \i in {1,...,#2} {
                 ($(\fixname.north east)!{\i/(#2+1)}!(\fixname.south east)$) +(-
\bbpenetrate,0) 
  coordinate (\fixname_out\i') -- +(\bbportlen,0) coordinate (\fixname_out\i)}\fi;
           }}}
     },
     bb name/.style={append after command={\pgfextra{\node[anchor=north] at 
(\fixname.north) {#1};}}}
}

\tikzset{
  	unoriented WD/.style={
  		every to/.style={draw},
  		shorten <=-\penetration, shorten >=-\penetration,
  		label distance=-2pt,
  		thick,
  		node distance=\spacing,
  		execute at begin picture={\tikzset{
  			x=\spacing, y=\spacing}}
  		},
  	pack size/.store in=\psize,
  	pack size = 8pt,
  	spacing/.store in=\spacing,
  	spacing = 8pt,
  	link size/.store in=\lsize,
  	link size = 2pt,
		penetration/.store in=\penetration,
		penetration = 2pt,
  	pack color/.store in=\pcolor,
  	pack color = blue,
  	pack inside color/.store in=\picolor,
  	pack inside color=blue!20,
  	pack outside color/.store in=\pocolor,
  	pack outside color=blue!50!black,
  	surround sep/.store in=\ssep,
  	surround sep=8pt,
  	link/.style={
  		circle, 
  		draw=black, 
  		fill=black,
  		inner sep=0pt, 
  		minimum size=\lsize
  	},
  	pack/.style={
  		circle, 
  		draw = \pocolor, 
  		fill = \picolor,
  		inner sep = .25*\psize,
  		minimum size = \psize
  	},
  	outer pack/.style={
  		ellipse, 
  		draw,
  		inner sep=\ssep,
  		color=\pocolor,
  	},
  	intermediate pack/.style={
  		ellipse,
  		dashed, 
  		draw,
  		inner sep=\ssep,
  		color=\pocolor,
  	},
}

\graphicspath{figures/}

\title{Monitoring and Diagnosability of Perception Systems}
\date{}
\author{Pasquale Antonante\qquad David I. Spivak\qquad Luca Carlone
\vspace{1em}\\
Massachusetts Institute of Technology \\
Cambridge, MA 02139 \\
{\tt\small \{antonap, dspivak, lcarlone\}@mit.edu}%
}

\begin{document}
\maketitle


\begin{mdframed}[
    tikzsetting={draw=red,ultra thick},
]
\center
An updated version of this paper can be found at \href{http://arxiv.org/abs/2011.07010}{http://arxiv.org/abs/2011.07010}
\end{mdframed}

\begin{abstract}

Perception is a critical component of high-integrity applications of robotics and autonomous systems, such 
as self-driving cars. 
 In these applications, failure of perception systems may put human life at risk, and 
 a broad adoption of these technologies relies on the development of methodologies to guarantee and monitor
  safe operation as well as detect and mitigate failures. 
Despite the paramount importance of perception systems, currently there is no formal 
approach for system-level monitoring. 
In this work, we propose a mathematical model for runtime monitoring and fault detection of perception systems.
 Towards this goal, we draw connections with the literature on self-diagnosability for multiprocessor systems, and 
generalize it to (i) account for \modules with heterogeneous outputs, and (ii) add a temporal dimension to the problem, which 
is crucial to model realistic perception systems where \modules interact over time.
 This contribution results in a graph-theoretic approach that, 
  given a perception system, is able to detect faults at runtime and allows computing 
  an upper-bound on the number of faulty \modules that can be detected.
Our second contribution is to show that the proposed monitoring approach can be elegantly described with the language of
 \emph{topos theory}, which allows formulating diagnosability over arbitrary time intervals.
\end{abstract}

\section{Introduction}\label{sec:introduction}

The automotive industry is undergoing a change that could revolutionize mobility.
Self-driving cars promise a deep transformation of personal mobility and have the potential to improve safety, 
efficiency (\eg commute time, fuel), and induce a paradigm shift in how entire cities are designed~\cite{Silberg12wp-selfDriving}.
One key factor that drives the adoption of such technology is the capability of ensuring 
and monitoring safe execution. 
Consider Uber's fatal self-driving crash~\cite{ntsbuber} in 2018: the report from the National Transportation Safety Board stated that ``inadequate safety culture'' contributed to the fatal collision between the autonomous vehicle and the pedestrian. 
 The lack of safety guarantees, combined with the unavailability of formal monitoring tools, 
  is the root cause of these accidents and has a profound impact on the user's trust.
The AAA's survey~\cite{aaa} shows that 71\% of Americans claim to be afraid of riding in a self-driving car.
This is a clear sign that the industry needs a sound methodology, embedded in the design process, to guarantee 
safety and build public trust.

While safety guarantees have been investigated in the context of control and decision-making~\cite{Mitsch17ijrr-verificationObstacleAvoidance,Foughali18formalise-verificationRobots}, 
the state of the art is still lacking a formal and broadly applicable methodology for monitoring \emph{perception systems}, which 
constitute a key component of any autonomous vehicle. Perception systems provide functionalities such as localization and obstacle mapping, lane detection, detection and tracking of other vehicles and pedestrians, among others.

\myParagraph{State of Practice}
The automotive industry currently uses 
five classes of methods to claim the safety of an autonomous vehicle (AV)~\cite{Shalev-Shwartz17arxiv-safeDriving}, namely: miles driven, simulation, scenario-based testing, disengagement, and proprietary.
The \emph{miles driven} approach is the most commonly used, and is based on the statistical argument that if the probability of crashes per mile is lower in autonomous vehicles than for humans, then autonomous vehicles are safer.
This approach is problematic since the meaningful%
\footnote{Driving on empty streets provides less evidence of safety than driving in dense urban environments.}
 number of fault-free miles that the AV should drive is significant---on the order of billions of miles \cite{Kalra16tra-selfDriving,Shalev-Shwartz17arxiv-safeDriving}---which would require years to achieve and would not enable frequent software updates.
 The same approach can be made more scalable through \emph{simulation}, but unfortunately 
  creating a life-like simulator is an open problem, for some aspects even more challenging than self-driving itself. 
 \emph{Scenario-based} testing is based on the idea
 that if we can enumerate all the possible driving scenarios that could occur, then we can simply expose the AV (via simulation, 
 closed-track testing, or on-road testing) to all of those scenarios and, as a result, be confident that the AV will only make safe driving decisions. 
However, enumerating all possible corner cases is a daunting (and system-dependent) task.
Finally, \emph{disengagement}~\cite{googledisengagements} is defined as the moment when a human safety driver has to intervene in order to prevent a hazardous situation. 
However, while less frequent disengagements indicate an improvement of the AV behavior, they do not give evidence of the system safety.

An established methodology to ensure safety throughout the life cycle of an automotive system is to use a \emph{standard} that every manufacturer has to comply with.
In the automotive industry, the standard ISO 26262 (Road vehicles - functional safety)~\cite{iso26262} is a risk-based safety standard that applies to electronic systems in production vehicles, including driver assistance and vehicle dynamics control systems.
A key issue is that ISO 26262 relies on the presence of human drivers (and mostly focuses on electronic systems rather than algorithmic aspects) hence it does not readily apply to fully autonomous vehicles~\cite{Koopman16sae-autonomousVehicleTesting}.
The recent ISO 21448 (Road vehicles - Safety of the intended functionality)~\cite{iso21448} is intended to complement ISO 26262.
It provides guidance on the applicable design, verification, and validation measures needed to achieve the safety at higher levels of automation.
The standard has been first publicly released in 2019, 
and---while stressing the key role of perception---it only provides high-level considerations to inform stakeholders.
In June 2019, 11 major stakeholder in the automated driving industry including Aptiv, Baidu, BMW, Intel, and Volkswagen published a report~\cite{AptivSafetyAV} on safety in autonomous driving.
The report focuses on highly autonomous vehicles, presenting an overview of safety by design and providing a sound discussion of the verification and validation of such systems.
 The report proposes a verification and validation approach mainly based on testing, and suggests the use of \emph{monitors} to detect 
off-nominal performance at the system and subsystem level (including neural networks).

Perception has been subject to increasing attention also  
outside the automated driving industry. 
For instance, autonomous aerial vehicles are expected to disrupt air mobility in the same way autonomous cars are disrupting urban mobility~\cite{UberElevate, EHangAirMobility}.
A recent report from the European Union Aviation Safety Agency (EASA)~\cite{EASADesignNN} acknowledges both the importance of such technologies and the lack of standardized methods to guarantee trustworthiness of such systems.
The report~\cite{EASADesignNN} considers an automatic landing system as case study and investigates the challenges in developing trustworthy AI, with focus on machine learning. 

These reports and standards stress the important role of perception for autonomous systems and 
motivate us to design a rigorous framework for system-level perception monitoring.

\myParagraph{Related Work} 
Formal methods~\cite{Ingrand19irc-verificationTrends} have been recently used as a tool to study safety of autonomous systems.
Formal methods use mathematical models to analyze and verify part of a program.
The insight here is that these mathematical models can be used to rigorously prove properties of the modeled program.
This approach has been successful for decision systems, such as obstacle avoidance~\cite{Mitsch17ijrr-verificationObstacleAvoidance} and road rule compliance~\cite{Roohi18arxiv-selfDrivingVerificationBenchmark}. 
This is mainly due to the fact that decisional specifications are usually model-based and have well-defined semantics~\cite{Foughali18formalise-verificationRobots}.
The challenge in applying this approach to perception is related to the complexity of modeling the physical environment~\cite{Seshia16arxiv-verifiedAutonomy}, and the trade-off between evidence for certification and tractability of the model~\cite{Luckcuck19csur-surveyFormalMethods}.

Relevant to the approach presented is this paper is the class of \emph{runtime verification} methods.
Runtime verification is an online approach based on extracting information from a running system and using it to detect (and possibly react to) observed behaviors satisfying or violating certain properties~\cite{Bartocci18book-runtimeVerification}.
Traditionally, the task of evaluating whether a \module is working properly was assigned to a monitor, which verifies some input/output properties on the \module alone.
Balakrishnan\setal~\cite{Balakrishnan19date-perceptionVerification} use Timed Quality Temporal Logic (TQTL) to reason about desiderable spatio-temporal properties of a perception algorithm.
Kang\setal~\cite{Kang18nips-ModelAF} use model assertions, which similarly place a logical constraint on the output of a \module to detect anomalies.
We argue that this paradigm does not fully capture the complexity of perception pipelines:
while monitors can still be used to infer the state of a single \module, perception pipelines provide an additional opportunity to cross check the \emph{compatibility} of results across different \modules. 
In this paper we try to address this limitation.

Performance guarantees for perception have been investigated for specific problems. 
In particular, related work on \emph{certifiable algorithms}~\cite{Yang20arxiv-teaser, Yang20cvpr-shapeStar, Briales18cvpr-global2view} provides algorithms that are capable of identifying faulty behaviors during execution. 
These related works mostly focus on specific algorithms, while our goal is to 
establish monitoring for perception \emph{systems}, including multiple interacting modules and algorithms.

\myParagraph{Contribution}
In this paper we develop a methodology to detect and identify faulty \modules in a perception pipeline at runtime.
In particular, we address two questions
(adapted from~\cite{Brundage20arxiv-trustworthyAI}):
\begin{inparaenum}[(i)] 
  \item \emph{Can I (as a developer) verify that the perception algorithm is providing reliable interpretations of the perception data?}
  \item \emph{Can I (as regulator) trace the steps that led to an accident caused by an autonomous vehicle?}
\end{inparaenum}
Towards this goal, we draw connections with the literature on self-diagnosability for multiprocessor systems, and 
generalize it to (i) account for \modules with heterogeneous outputs, and (ii) add a temporal dimension to the problem to account for \modules interacting over time.
This contribution results in a graph-theoretic approach that, 
  given a perception system, is able to detect faults at runtime and allows computing 
  an upper-bound on the number of faulty \modules that can be detected.
This upper bound is related to the level of redundancy within the system and provides a quantitative measure of robustness.
Our second contribution is to show that the proposed monitoring approach can be elegantly described with the language of
 \emph{topos theory}~\cite{MacLane12book-toposTheory,Johnstone02book-toposTheory} and allows formulating diagnosability and faults detection over arbitrary time intervals.


\section{An Example of Perception Pipeline}

Consider the system depicted in~\cref{fig:localization_pipeline}: this is an example of 
perception system used to localize an autonomous vehicle (AV) using multiple sensor streams. 
We are going to use this system as a running example throughout the paper to elucidate our mathematical model.
\begin{figure}[ht!]
  \includegraphics[width=\textwidth]{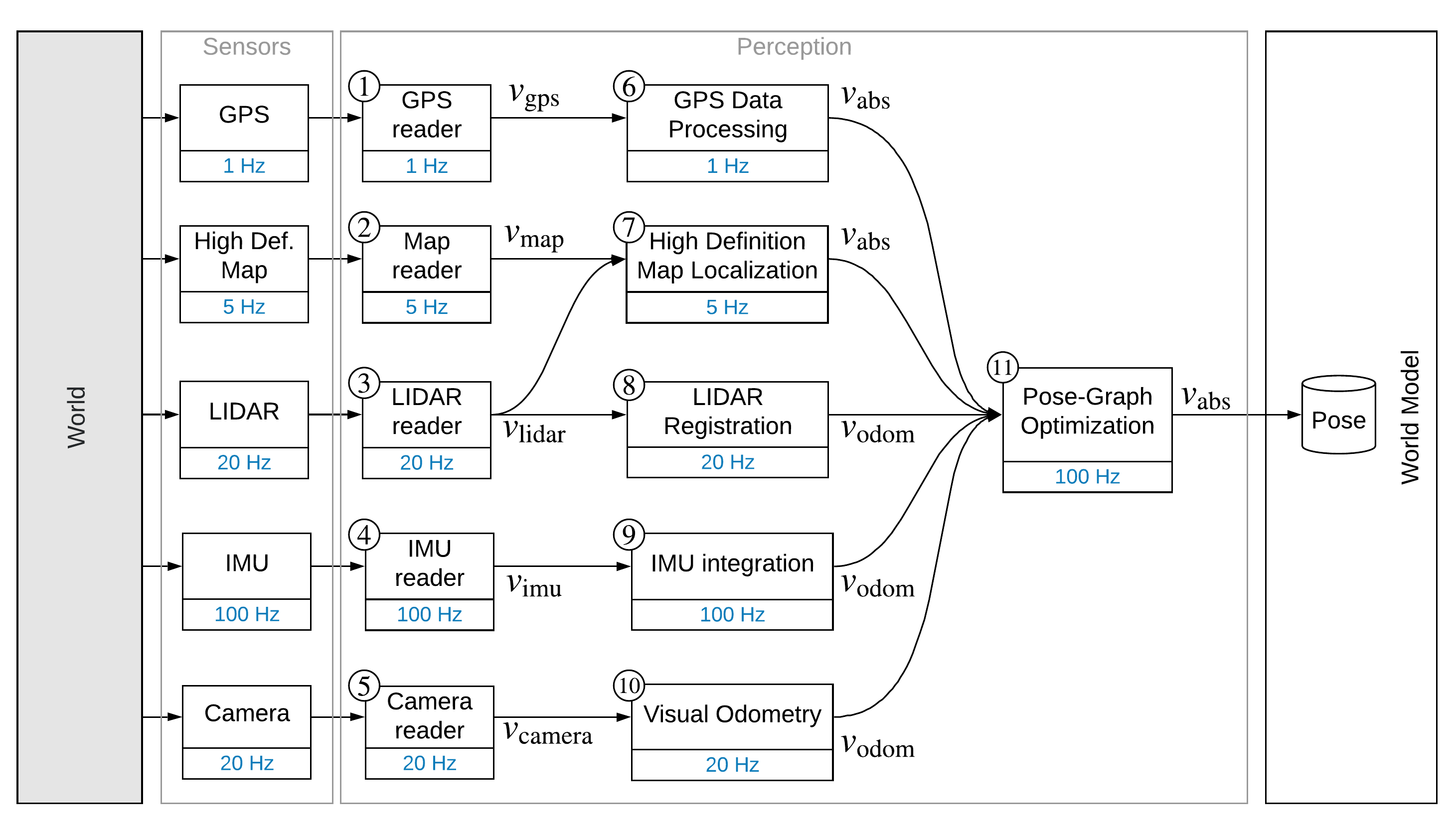}\vspace{-4mm}
  \centering
  \caption{Example of localization system for Autonomous Vehicles.}
  \label{fig:localization_pipeline}
\end{figure}

In~\cref{fig:localization_pipeline}, the AV observes the world through noisy sensors, each of which provides information to various perception \modules.
These perception \modules collect, organize, and interpret sensory information in order to create a model of the world.
In this paper, for the sake of simplicity, we consider a world model that includes only the absolute pose---position and heading---of the AV, 
an element of the 3-dimensional Special Euclidean group $\SEthree$.

In~\cref{fig:localization_pipeline},  each rectangle inside the ``Perception'' box represents a \emph{perception \module}, while each arrow represents a data stream.
Let $\setVars$ denote a set of \emph{variables}; in our example, $\setVars = \{v_{\mathrm{gps}},v_{\mathrm{map}},v_{\mathrm{lidar}},v_{\mathrm{imu}},v_{\mathrm{camera}},v_{\mathrm{abs}},v_{\mathrm{odom}}\}$.
We can tag each arrow with the variable it carries. 
Moreover, each variable has its own type.
Let $\Type:\setVars\to\smset$ be a map; we say that $\Type(v)$ is the type of variable $v$.
For example, $\Type(v_{\mathrm{imu}})=\Real{6}$, a real vector of 6 elements listing the acceleration and angular velocities measured
 by the Inertial Measurement Unit (IMU).
At any instant of time, the system takes the measurements from the sensors and estimates 
the AV's absolute pose, denoted by $v_{\mathrm{abs}}$. Let us analyze the \modules in~\cref{fig:localization_pipeline}:
\begin{itemize}
  \item The Global Positioning System (GPS) provides raw latitude and longitude data to the \emph{GPS reader}, which packages this information and sends it to \emph{GPS Data Processing}, which in turn estimates the car's absolute pose (variable $v_{\mathrm{abs}}$).
  \item The Light Detection and Ranging (LIDAR) sensor uses light in the form of a pulsed laser to measure reflection distances, producing a 3D point cloud of the surrounding environment; the \emph{LIDAR reader} gets the point cloud while the \emph{LIDAR Registration} \module
estimates the relative motion between consecutive time instants by comparing consecutive scans~\cite{Yang20arxiv-teaser},
The estimate of the relative motion of the AV is commonly referred to as ``odometry'' in robotics~\cite{Cadena16tro-SLAMsurvey} (variable $v_{\mathrm{odom}}$).
  \item The \emph{High Definition Map Localization} uses a given high definition map of the environment to estimate the absolute position (variable $v_{\mathrm{abs}}$) of the vehicle by comparing the LIDAR scan with the map.
  \item The \emph{Cameras} provide images of the environment, the \emph{Camera reader} collects raw images, 
  while 
  the \emph{Visual Odometry} module estimates the relative motion of the vehicle by comparing two consecutive images (variable $v_{\mathrm{odom}}$).
  \item The \emph{Inertial Measurement Unit} (IMU) measures accelerations and angular velocities of the car, then by integrating these values over time, \emph{IMU integration} estimates the relative motion of the vehicle (variable $v_{\mathrm{odom}}$).
  \item Finally, the \emph{Pose-Graph Optimization}~\cite{Cadena16tro-SLAMsurvey} \module estimates the vehicle's pose 
  by fusing the noisy absolute and relative pose measurements produced by the other \modules.
 \end{itemize}

Errors are common throughout this pipeline, so it is crucial for the system to self-diagnose these errors in real time, discard outliers, and produce reliable estimates of the car's pose (which can be then used for planning and control purposes). 
For instance, when operating near tall buildings, the GPS can provide unreliable measurements. Similarly, 
when used in crowded scenes, visual odometry can provide incorrect motion estimates. 
In the next sections we study how faults can be diagnosed, using this perception system as an example.

\section{Monitoring of Perception Systems}\label{sec:monitoring}

Our first contribution is to develop a mathematical model for monitoring and fault diagnosis in perception systems.
Towards this goal, we draw connections with the literature on fault diagnosis in multiprocessors systems, 
which has been extensively studied since the late 1960s.
Section~\ref{sec:diagnosability} reviews the notion of diagnosability and existing results for fault detection and to characterize 
diagnosable systems using the PMC (Preparata, Metze, and Chien) model~\cite{Preparata67tec-diagnosability}. 
Section~\ref{sec:diagnosabilityPerception} extends the notion of diagnosability to perception systems and discusses how to 
model temporal aspects arising in real AV applications.

\subsection{Diagnosability}\label{sec:diagnosability}
In the PMC~\cite{Preparata67tec-diagnosability}  model, a set of independent \emph{processors} 
is assembled such that each processor has the capability to communicate with a subset of the other processors, and all of the processors perform the same computation.
In such system a fault occurs whenever the outputs of a processor differs from the outputs of fault-free \units; the problem is to identify which is which.
Each processor is assigned a particular subset of the other processors for the purpose of testing.
Using a comparison-based mechanism, the model aims to characterize the set of faulty processors.
Clearly, it is not possible to determine the faulty subset in general, so much of the literature on multiprocessor diagnosis considers two fundamental questions~\cite{Dahbura84tc-diagnosability}:
\begin{inparaenum}[(i)] 
  \item Given a collection of \units and a set of tests, what is the maximum number of arbitrary \units which can be faulty such that the set of faulty \units can be uniquely identified?
  \item Given a set of test results, does there exist an efficient procedure to identify the faulty \units?
\end{inparaenum}
The key tool to address these questions is the \emph{diagnostic graph}~\cite{Preparata67tec-diagnosability}.

\myParagraph{Diagnostic Graph}
At any given time, each \unit is assumed to be in one of two states: \emph{faulty} or \emph{fault-free}. 
Diagnosis is based on the ability of \units to test---\ie to provide an opinion about the faultiness of---other \units.
Formally we assume that each \unit implements one or more \emph{consistency functions}; these are Boolean functions that return \emph{pass} ($0$) or \emph{fail} ($1$), depending on whether the output of two \units is in agreement. 

To perform the diagnosis, we follow~\cite{Preparata67tec-diagnosability} and model the problem as a directed graph $D=( U ,E)$, where $ U $ is the set of \units, 
while the edges $E$ represent the test assignments.
In particular, for an edge $(\uniti,\unitj)\in E$, we say that node $\uniti$ is testing node $\unitj$.
The outcome of this test is the result of the consistency functions of $\uniti$, in other words $1$ (resp. $0$) if $\uniti$ evaluates $\unitj$ as faulty (resp. fault-free).
Fault-free \units are assumed to provide correct test results, whereas no assumption is made about tests executed by faulty \units: they may produce correct or incorrect test outcomes. 
We call $D$ the \emph{diagnostic graph}. \toAdd{No \unit can test itself.}

The collection of all test results for a test assignment $E$ is called a \emph{syndrome}.
Formally, a syndrome is a function $\sigma:E\to\{0, 1\}$.
The syndrome is processed by an external entity which diagnoses the system, that is, 
makes some determination about the faultiness status of every \unit in the system.
The notion of $t$-diagnosability formalizes when it is possible to use the syndrome to diagnose and identify faults in a system.

\begin{definition}[$t$-diagnosability~\cite{Preparata67tec-diagnosability}]\label{def.t_diag}
  A diagnostic graph $D=( U ,E)$ with $|U|=n$ \units is \emph{$t$-diagnosable} ($t<n$) if, given any syndrome, all faulty \units can be identified, provided that the number of faults presented does not exceed $t$.
\end{definition}

What makes fault detection challenging is the fact that we can have multiple \units that provide incorrect results but are 
consistent with each other. This is formalized by the notion of consistent fault set.

\begin{definition}[Consistent fault set~\cite{Preparata67tec-diagnosability}]\label{def:cfs}
For a syndrome $\sigma$, denote the test outcome assigned to the edge $(\uniti,\unitj)$ by $\sigma_{\uniti\unitj}\in\{0,1\}$. 
 For a $t$-diagnosable graph $D=( U ,E)$ and a syndrome $\sigma$, a subset $ F\subseteq U $ is a \emph{$t$-consistent fault set} 
  if and only if 
 \begin{enumerate}[label=(\roman*)]
   \item $| F|\leq t$;
   \item if $\sigma_{\uniti\unitj} = 1$ then either $\uniti \in F$ or $\unitj\in F$\label{itm:cfs_faulty}
   \item if $\uniti,\unitj \in U \setminus F$ then $\sigma_{\uniti\unitj}=0$.\label{itm:cfs_faultfree} 
 \end{enumerate}
\end{definition}
\cref{itm:cfs_faulty} in \cref{def:cfs} says that if $\sigma_{\uniti\unitj}=1$ ($\uniti$ disagrees with $\unitj$) then 
at least one between $\uniti$ and $\unitj$ must be outside the consistent fault set. 
\cref{itm:cfs_faultfree} states that \units outside the faulty set are in agreement with each other; note that
the condition is not \emph{if and only if} because we are assuming that tests performed by faulty nodes are unreliable (see example in \cref{fig:example1}), therefore we can have $\sigma_{\uniti\unitj}=0$ when $\uniti$ is faulty and returns a wrong test result.

\myParagraph{Example}
Consider the simple example in \cref{fig:example1} (adapted from \cite{Preparata67tec-diagnosability}).
For the graph in \cref{fig:example1}, a syndrome will be represented as a list of five elements, $(\sigma_{12}, \sigma_{23}, \sigma_{34}, \sigma_{45}, \sigma_{51} )$.
Assume exactly one of the \units, say \node{1} (w.l.o.g), is faulty. Then 
$\sigma_{23} = \sigma_{34} = \sigma_{45} = 0$ and $\sigma_{51} = 1$,
\ie \node{5} correctly identifies \node{1} as faulty and $\sigma_{12}$ could be either 0 or 1: being a faulty node, \node{1} may or may not correctly diagnose \node{2}; in the figure we label that edge by $x$.
In this simple case it is easy to check that, assuming there is only one faulty \unit, we can correctly identify that the faulty \unit is \node{1}.
It is natural to ask: can we identify a faulty condition with two faulty \units?
In this graph, we cannot: to prove that, it is sufficient to show that two different faulty conditions generate the same syndrome (\cref{fig:example1c}).
\vspace{-2em}
\begin{figure}[h]
  \centering
  \subfloat[][Diagnostic graph]{
    \includegraphics[width=0.30\textwidth]{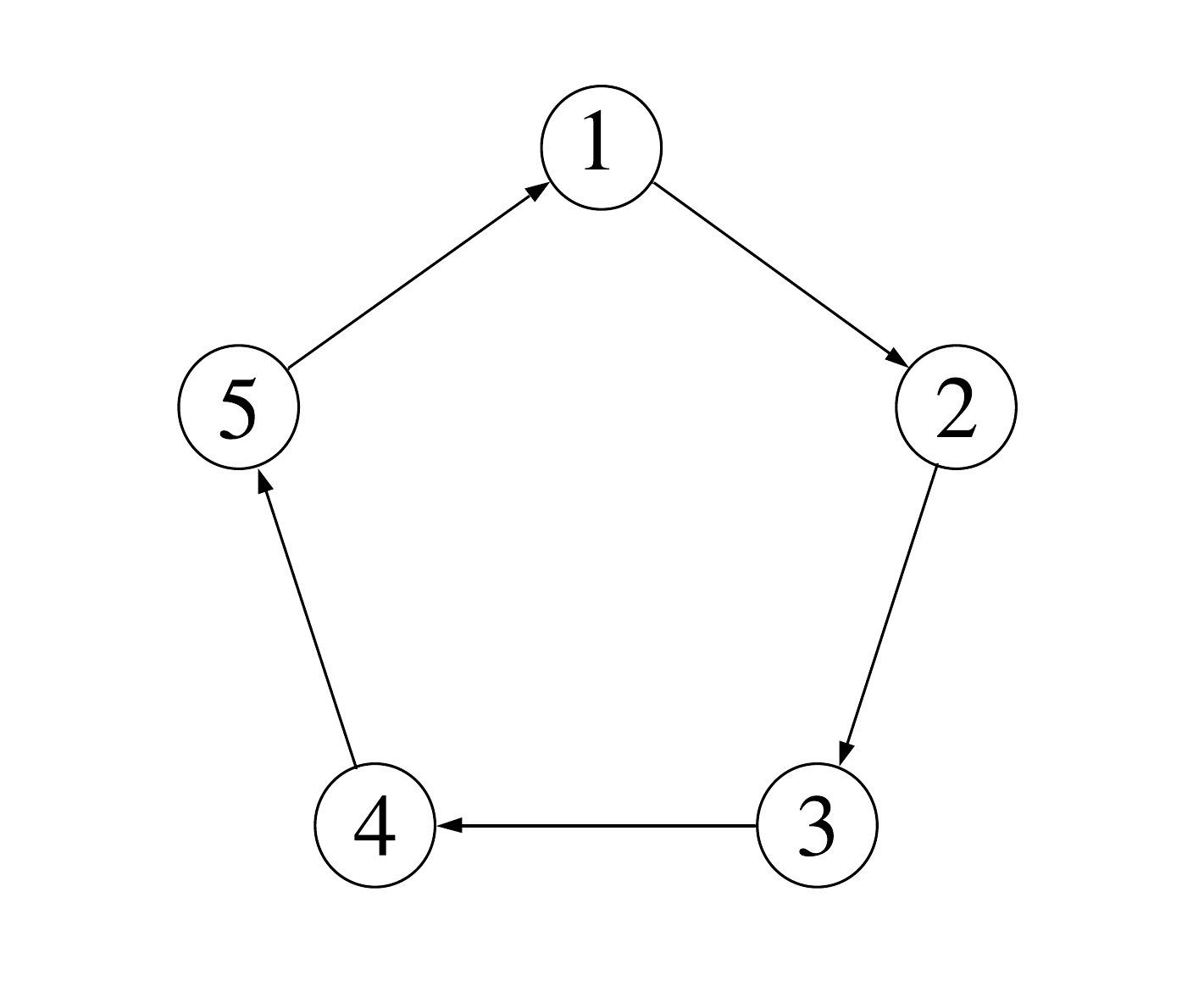}
    \label{fig:example1a}
  }
  \subfloat[][Single fault]{
    \includegraphics[width=0.30\textwidth]{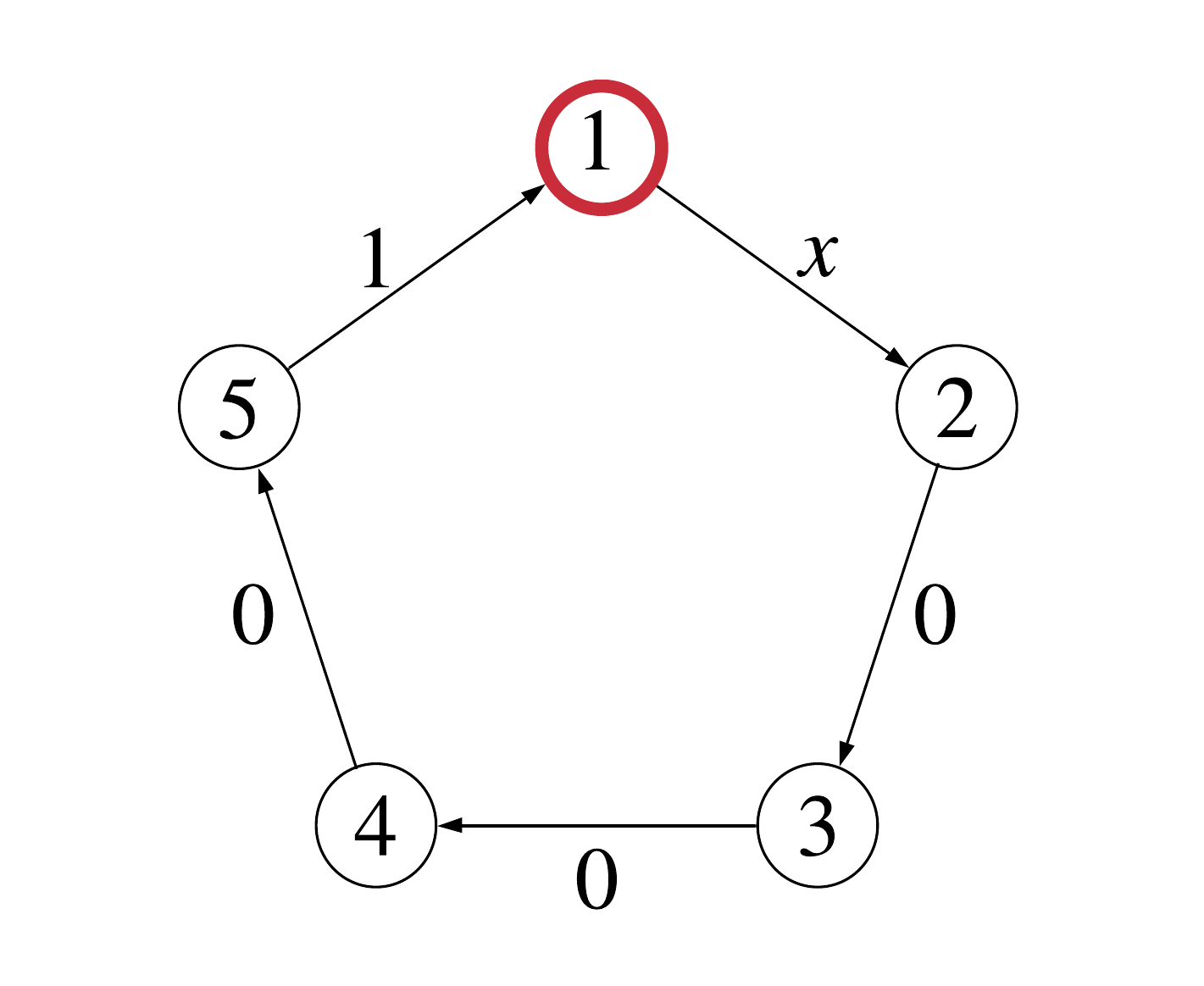}
    \label{fig:example1b}
  }
  \subfloat[][Double fault]{
    \includegraphics[width=0.30\textwidth]{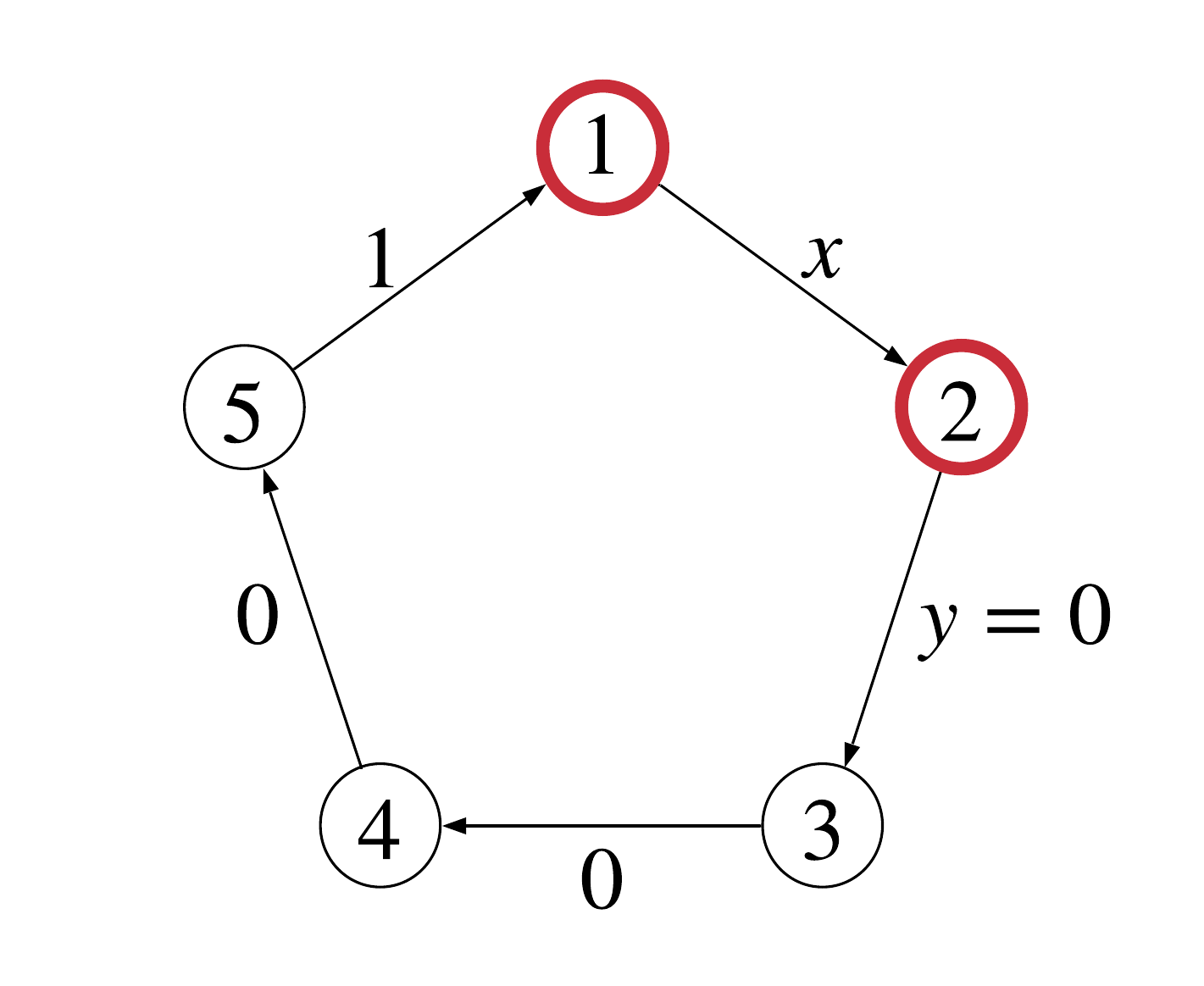}
    \label{fig:example1c}
  }
  \caption{(a) Example of diagnostic graph with five \units and (b)-(c) two faulty conditions (faulty processors shown in red) exhibiting the same syndrome.}
  \label{fig:example1}
\end{figure}

The problem of determining the maximum value of $t$ for which a given system is $t$-diagnosable is called the \emph{diagnosability problem}.
We denote the maximum value of $t$-diagnosability of a graph $D$ by $t(D)$.

\subsubsection{Characterization of $t$-diagnosability}
Consider a directed graph $D=( U ,E)$. For $\uniti \in U $, we denote by $\delta_{\tn{in}}(\uniti)$ the number of edges directed toward $\uniti$ (in-degree). We denote by $\delta_{\tn{in}}(D) = \min_{\uniti\in U } d_{\tn{in}}(\uniti)$ the minimum in-degree of the graph.
Moreover, we denote $\Gamma(\uniti) = \{\unitj \in U \mid (\uniti,\unitj) \in E\}$ the \emph{testable} set of $\uniti$ (outgoing neighbors 
of $\uniti$).
Finally, for $ X\subset U$, denote $\Gamma(X) = \{\bigcup_{\uniti \in X} \Gamma(\uniti) \setminus X\}$.

We can now state the following fundamental theorem.
\begin{theorem}[Characterization of $t$-diagnosability~\cite{Hakimi74tc-diagnosability}]\label{thm:pcm}
Let $D=( U ,E)$ be a diagnostic graph with $| U |=n$ \units. Then $D$ is $t$-diagnosable if and only if
\begin{enumerate}[label=(\roman*)]
    \item\label{itm:t_ub} $n\geq 2t+1$;
    \item\label{itm:t_lb} $\delta_{\tn{in}}(D) \geq t$; and
    \item\label{itm:t_p} for each integer $p$ with $0\leq p<t$, and each $ X\subset U $ with $| X|=n-2t+p$ we have $|\Gamma( X)| > p$.
\end{enumerate}
\end{theorem}

Using \cref{thm:pcm} we can quickly find that the maximum $t$ in the graph in \cref{fig:example1} is $t=1$ because
$ t\leq \lfloor \frac{n-1}{2} \rfloor = 2 $ and $t\leq \delta_{\tn{in}}(D) = 1$.
 Moreover, condition \ref{itm:t_p} holds ($t=1$ forces $p=0$ and each $X\subseteq U$ with $|X|=3$ has $|\Gamma(X)| > 0$).

A naive implementation of \cref{thm:pcm} would lead to an algorithm of time complexity $O(n^{t+2})$~\cite{Bhat82acm-diagnosability}.
Bhat~\cite{Bhat82acm-diagnosability} propose an improved algorithm to find the maximum $t$ of an arbitrary graph in $O(n^{u+2} \log_2(u))$ where $u=\min \{ \delta_{\tn{in}}(D), \floor{(n-1)/2}\}$.
Since this approach can be impractical for large graphs, they also propose a polynomial algorithm to find a suboptimal value of $t$ in $O(|U|^{3/2} +|E|)$.

\subsubsection{Fault Detection}
Once we have a syndrome on a $t$-diagnosable graph, we would like to actually identify the set of faulty \units, provided that the number of faults does not exceed $t$.
Dahbura\setal~\cite{Dahbura84tc-diagnosability} showed that the problem of identifying the set of faulty \units is related to the problem of finding the minimum vertex cover set of an undirected graph (i.e.\ the smallest set $C$ of vertices such that every edge has an endpoint in $C$).
In general, the problem of finding a minimum vertex cover is in the class of NP-complete problems, meaning that there is no known deterministic algorithm that is guaranteed to solve the problem in polynomial time, but the validity of any solution can be tested in polynomial time.
However, the work~\cite{Dahbura84tc-diagnosability} exploits some special properties of $t$-diagnosable graphs to propose an algorithm with time complexity $O(n^{2.5})$ for fault identification.
In a later work, Sullivan~\cite{Sullivan88tc-diagnosability} proposes an algorithm with time complexity $O(t^3+|E|)$ for fault identification and proved that this is the tightest bound if $t$ is $o(n^{5/6})$.
These results ensure us that fault identification is practical in real-time and scales well with bigger diagnostic graphs.

\subsection{Diagnosability for Perception}\label{sec:diagnosabilityPerception}
In this section we generalize the approach of~\cref{sec:diagnosability} to make it applicable to perception systems.
In particular, while in~\cref{sec:diagnosability} 
multiple processors perform the same computation and yield the same type of output, 
perception systems are characterized by the fact that each \module produces different variables and potentially at different rates (\cref{fig:localization_pipeline}).
Therefore, in order to build on the diagnosability framework described in the previous section, we have to 
\begin{inparaenum}[(i)] 
  \item define consistency functions for a perception system with arbitrary \modules, and 
  \item capture temporal aspects that are missing in the original PMC model~\cite{Preparata67tec-diagnosability}
\end{inparaenum}

\begin{figure}[t]
  \includegraphics[width=0.7\linewidth]{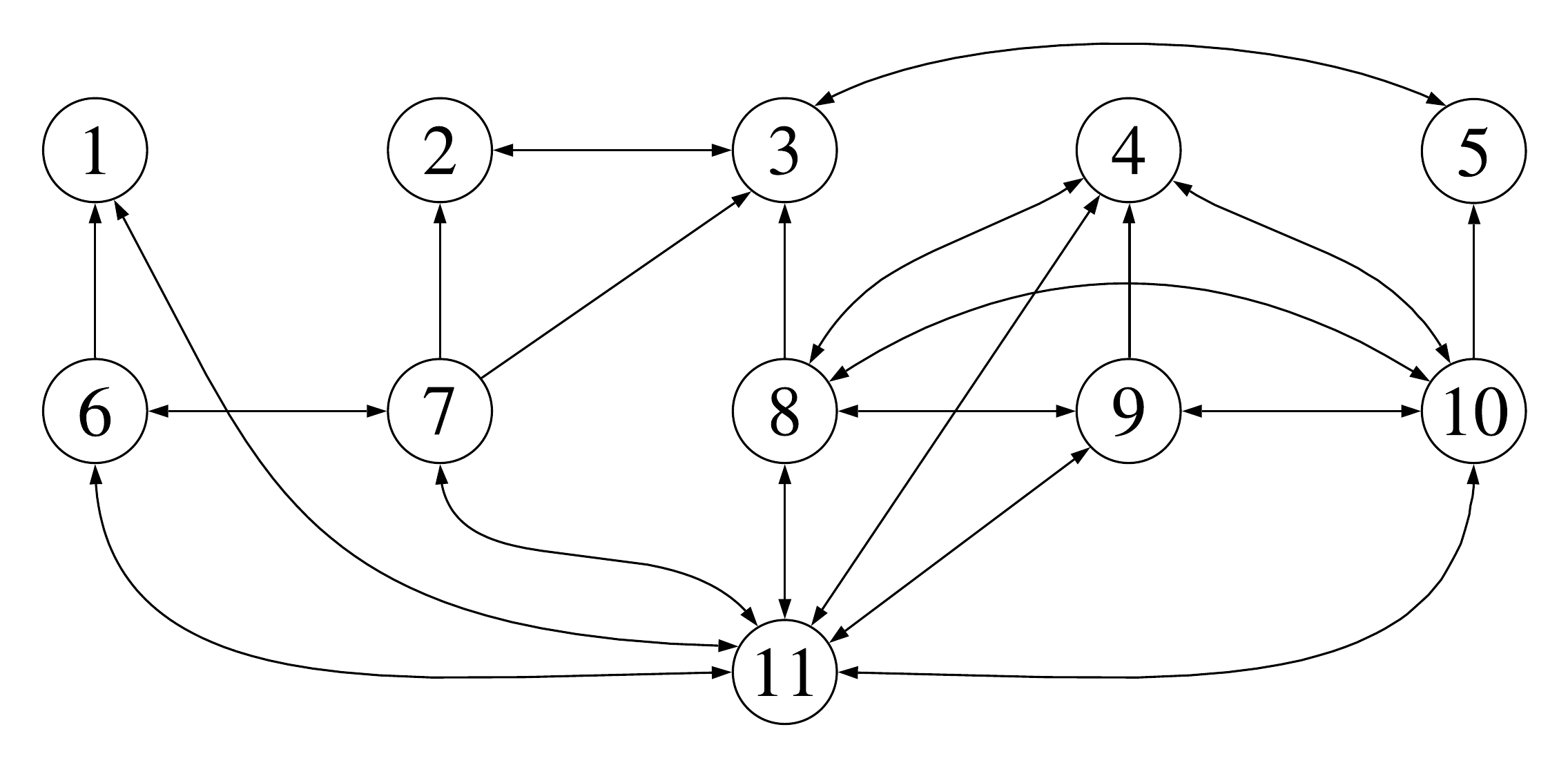}
  \centering
  \caption{Example of diagnostic graph $D=(U ,E)$ for the localization pipeline in~\cref{fig:localization_pipeline}.}
  \label{fig:diagnostic_graph}
\end{figure}

\subsubsection{Perception Consistency Functions}

Consistency functions are used to implement pairwise tests for each edge in a diagnostic graph.
Due to the heterogeneity of the \modules of a perception system, we can have different types of consistency functions, which 
we describe in the rest of this section. Using these consistency functions, we can build a diagnostic graph: 
for instance, in~\cref{fig:diagnostic_graph} we report an example of diagnostic graph for the localization example in~\cref{fig:localization_pipeline}. 

\myParagraph{Input and Output admissibility}
Consider the case in which the camera captures a very dark image caused by a sudden change of brightness in the scene.
Upon receiving camera images, the Visual Odometry \module can detect the underexposed image as unintended and report a failure
 from testing the camera \module.
If we look at the example in \cref{fig:localization_pipeline}, we can capture this test as an edge between \modules $\node{10}\to\node{5}$.
In perception systems, this type of consistency function is very common.
Following the same reasoning we can add a consistency checks between any \module and its predecessor, for example, if the point cloud provided by the LIDAR (\node{3}) to LIDAR Registration (\node{8}) contains $N$ coincident points we clearly have a fault in the LIDAR \module. These edges are added to the diagnosis graph in~\cref{fig:diagnostic_graph}. 
Note that using this logic we can also monitor output admissibility (the output of a \module will be the input of another). 

\myParagraph{Input Consistency}
Consider the IMU integration \module, suppose it receives data from three IMUs, the \module can compare the three noisy measurements and detect if one of the three does not agree with the others; for example an IMU provides very different acceleration measurements. 
Hence the IMU integration \module can test the three IMU reader modules.
\toAdd{this is fishy: not true if you only have 2 edges. We should generalize to a factor graph/hypergraph}

\myParagraph{Output Consistency}
Consider the GPS reader and the Map reader.
They don't directly share information but suppose that a GPS measurement return a latitude and longitude that the map predicts to be occupied by water.
The map reader, using this information, can identify the faulty measurement of the GPS, returning fail whenever testing the GPS.
In our perception example, this test correspond to an edge $\node{2}\to\node{1}$ and/or $\node{2}\to\node{6}$.

\myParagraph{Input/Output Consistency}
Consider the case in which Visual Odometry is not able to correctly estimate the pose of the vehicle.
After fusing all the incoming pose measurements, the Pose-Graph Optimization \module \node{11}, can identify that visual odometry produced 
an off-nominal value~\cite{Lajoie19ral-DCGM}. 
Input/Output consistency differs from Input consistency in that it needs to compute the output variable in order to test the validity of the inputs.
In our system this example corresponds to the edges $\node{11}\to\node{6}$, $\node{11}\to\node{7}$, \ldots, $\node{11}\to\node{10}$.

\myParagraph{Example}
Each \module in~\cref{fig:localization_pipeline} can use one or more of these consistency functions described above to
test other \modules in the system. 
Each test can detect a subset of failures that might occur, so the combination of all available tests maximize the set of detectable faults in the system. In~\cref{fig:diagnostic_graph}, we report an example of diagnostic graph for the localization example in~\cref{fig:localization_pipeline}. 
By inspection and using condition (ii) in~\cref{thm:pcm}, we know that $t\leq 2$ because \node{6} has in-degree equal to $2$.
By checking condition (iii) we discover that the graph is in fact only $1$-diagnosable.
In particular, for $t=2$, $p=1$ and $ X=\{1, 2, 3, 4, 5, 8, 9, 10\}$ ($|X|=11-4+1=8$) we have that $\Gamma(X) = \{11\}$ therefore $|\Gamma( X)|\leq p$.
This means we can uniquely identify faulty modules only in the case where just one such fault occurs.

\subsubsection{Temporal Diagnostic Graphs}

In the previous section, we considered instantaneous fault diagnosis. In other words, at any iteration, given a diagnostic graph and a syndrome we can detect and identify faulty \units only considering tests occurring at that time instant. 
However, perception \modules evolve over time and considering the temporal dimension allows adding temporal checks 
and modeling systems with \modules operating at multiple frequencies.
In the following, we extend the notion of diagnostic graphs to account for the temporal dimension of perception. 

In a perception pipeline, different \modules publish information at different frequencies.
For example consider the system in~\cref{fig:localization_pipeline}. Within each second, the IMU publishes 100 times, whereas the GPS reader only publishes once.
We can thus take the subgraph $D'\ss D$ on nodes 
according to minimum publishing frequencies and with edges that have both endpoints in $D'$.
For instance, consider the nodes that publish at least at \SI{100}{\hertz}, those nodes are \node{4}, \node{9}, and \node{11}. 
The resulting subgraph of $D$ that we would obtain taking these nodes in depicted in~\cref{fig:diagnostic_graph_100hz}. 
Note that we might have different subgraphs for each publishing frequency as in~\cref{fig:diagnostic_graphsByRate}. 
Clearly we can perform fault detection over each of these ``instantaneous'' diagnostic graphs. However, as mentioned above, 
considering the evolution of the system over time provides further opportunities for fault detection.
\begin{figure}[h]
  \centering
  \subfloat[][\SI{100}{\hertz}]{
    \includegraphics[width=0.12\textwidth]{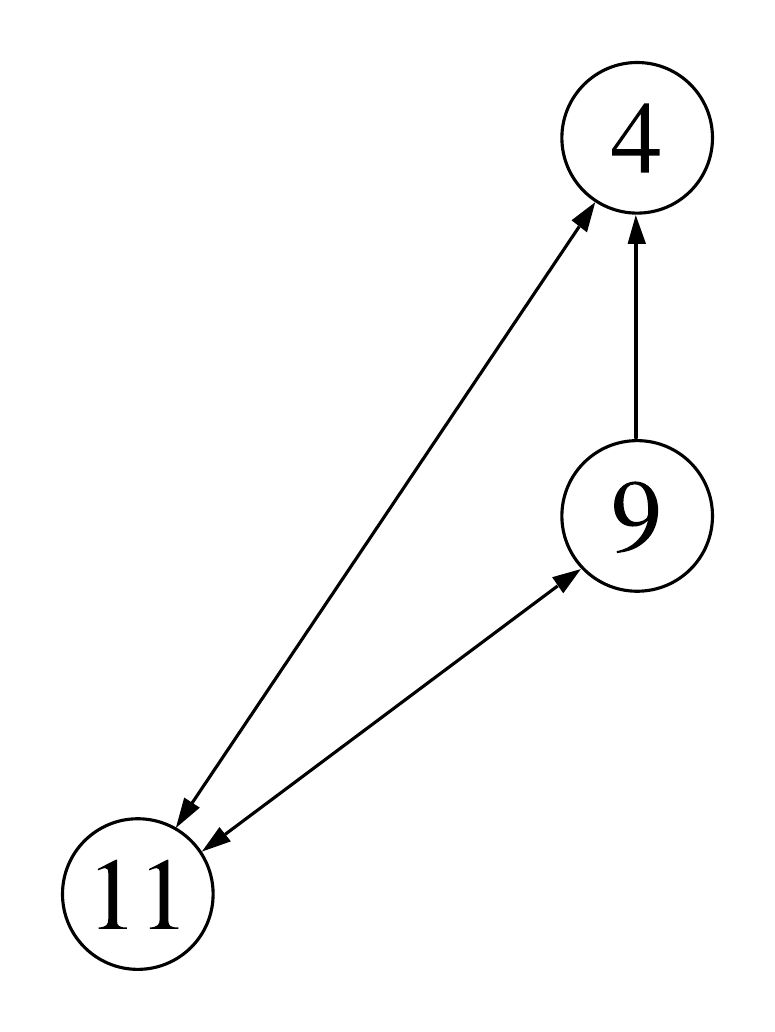}
    \label{fig:diagnostic_graph_100hz}
  }
  \subfloat[][\SI{20}{\hertz}]{
    \includegraphics[width=0.2\textwidth]{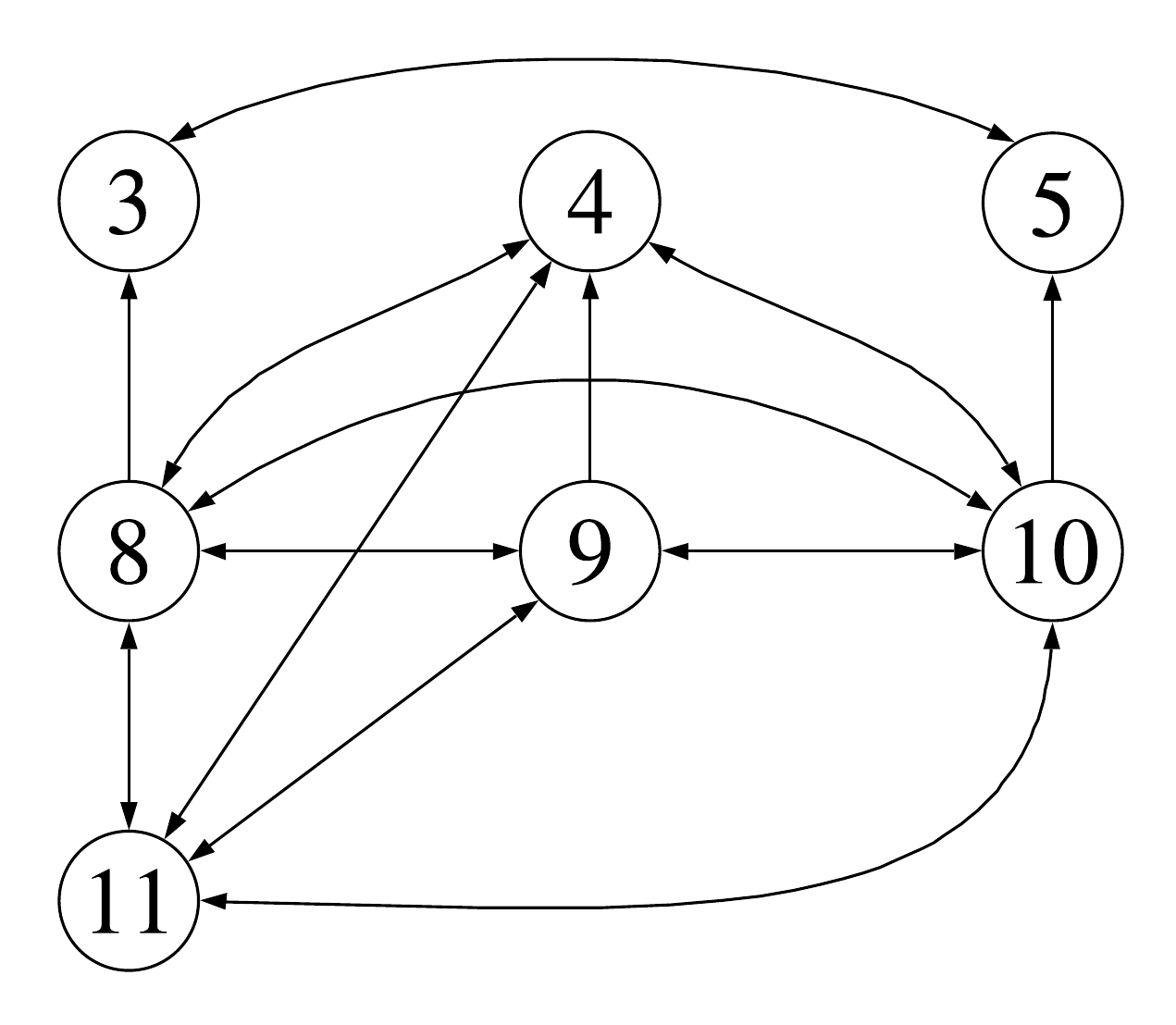}
    \label{fig:diagnostic_graph_20hz}
  }
  \subfloat[][\SI{5}{\hertz}]{
    \includegraphics[width=0.25\textwidth]{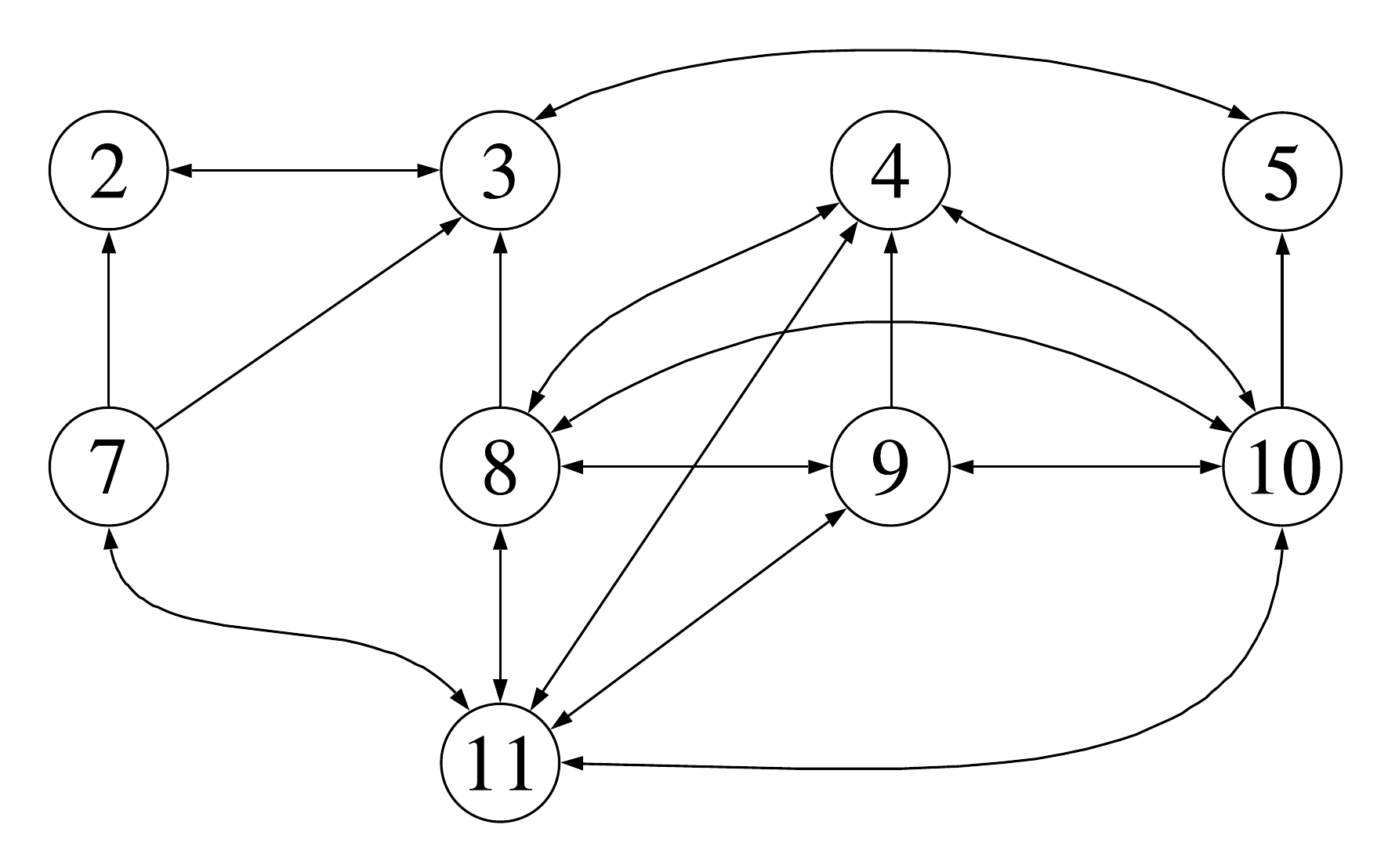}
    \label{fig:diagnostic_graph_5hz}
  }
  \subfloat[][\SI{1}{\hertz}]{
    \includegraphics[width=0.35\textwidth]{figures/diagnostic_graph.pdf}
    \label{fig:diagnostic_graph_1hz}
  }
  \caption{Subgraphs of the diagnostic graph in~\cref{fig:diagnostic_graph}, one for each publishing frequency.}
  \label{fig:diagnostic_graphsByRate}
\end{figure}

If we consider an arbitrary interval $[a,b]$ of time (taken in seconds), over the interval we have the opportunity to collect multiple diagnostic graphs with their syndromes (\cref{fig:dstacks}). 
Furthermore, the output of the modules of these graphs cannot be arbitrary, and must have some temporal consistency.
For instance, 
because motion is continuous, measurements of the AV pose at consecutive times should be close to one another.
 Another example is that, since the AV has a limited braking power (imposed by physical constraints), a set of consecutive IMU acceleration measurements cannot exceed some bounds.
Therefore the diagnostic graphs at different moments in time can be connected via temporal consistency function, \ie we can make tests (or edges) that go across time. We call the collection of the diagnostic graphs in the time interval $[a,b]$ 
and the corresponding temporal edges a {\bf $[a,b]$-temporal diagnostic graph} (\cref{fig:dstacks}).
The simplest form of \emph{temporal consistency} involves an edge connecting two identical nodes occurring at consecutive time steps, 
\eg two consecutive IMU measurements.
However it can easily be extended to different \modules and neighboring (but not adjacent) time steps.

Besides being more expressive, 
this approach has the potential to greatly increase the $t$-diagnosability of the system.
Intuitively, the temporal diagnostic graph becomes larger, but also more diagnosable, when longer intervals of time are considered.
\begin{corollary}\label{cor:more_edges}
If $D\ss D'$ are two graphs with the same set of vertices, then $t(D)\leq t(D')$.
\end{corollary}
\begin{proof}
Assume conditions (i), (ii), and (iii) hold for $\pair{D,t}$.
Since $D$ and $D'$ have the same number of \units, condition (i) holds for $\pair{D',t}$.
Since $\delta_{\mathrm{in}}(D)\leq\delta_{\mathrm{in}}(D')$ and for every $X\ss U$ we have $\Gamma_D(X)\leq\Gamma_{D'}(X)$, conditions (ii) and (iii) hold also for $\pair{D',t}$.
\end{proof}
We can prove that the $t$-diagnosability is monotonic under temporal restriction.
\begin{proposition}\label{prop:subgraph_diag}
Let $D=(\setU,\setE)$ be a diagnostic graph. 
Suppose given subsets $\setU_i\ss\setU$ that cover $\setU$, in the sense that $\setU=\cup_{i\in I}U_i$, and for each $i\in I$, let $D_i\ss D$ be the largest subgraph with vertices $\setU_i$.
Then if each $D_i$ is $t$-diagnosable, so is $D$.
\end{proposition}
\begin{proof}
Assume each $D_i$ is $t$-diagnosable, and suppose that $\sigma$ is a syndrome for a $t$-consistent fault set on $D$ (so it has at most $t$-many faulty nodes).
Then its restriction to each $D_i$ has fewer than $t$-many faulty nodes, so the faulty nodes in $D_i$ can be accurately diagnosed.
Since every node is in some $D_i$, every node can be accurately diagnosed, so $D$ is $t$-diagnosable.
\end{proof}

\myParagraph{Example}
Let us consider a small interval of time, say $[a, a+0.02]$.
In such an interval of time we should get three different diagnostic graphs for the \SI{100}{\hertz} rate (see \cref{fig:diagnostic_graph_100hz}).
Each of those graph is $1$-diagnosable.
Now suppose we can test consistency across one and two consecutive instants, obtaining the temporal diagnostic graph in \cref{fig:dstacks}.
The diagnosability of the temporal diagnostic graph in \cref{fig:dstacks} is now $3$-diagnosable (\cf \cref{thm:pcm}).
From this example, one can see the advantages of including the temporal dimension for fault diagnosis.
\begin{figure}
  \centering
  \includegraphics[height=5cm]{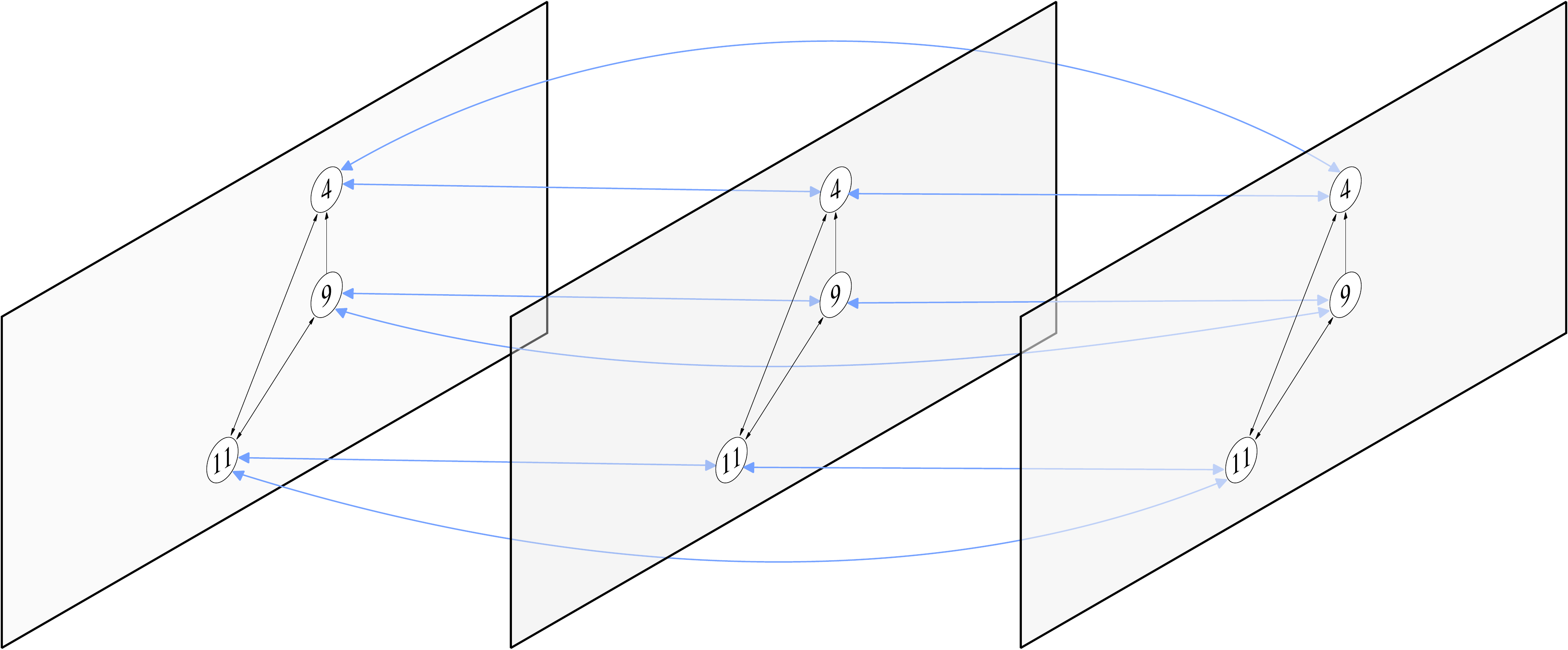}
  \caption{Example of \emph{temporal diagnostic graph} over a small time interval (\SI{0.02}{\second}).}
  \label{fig:dstacks}
\end{figure}

\section{Temporal Type Theory}

This section shows that temporal diagnostic graphs can be elegantly described with the language of \emph{topos theory}.
In particular, in order to organize the information of multiple graphs that interact over time-intervals of varying length, as well as the diagnosability and collection of faulty nodes in these graphs, we use the mathematical language of \emph{temporal type theory} as described in~\cite{Schultz19book-ttt}. 
In the following,~\cref{sec:ttt} briefly reviews basic concepts of topos theory, while~\cref{sec:tdgInterpretation} tailors these concepts to temporal diagnostic graphs.

\subsection{Sheaves, Behavior Types, and Toposes}\label{sec:ttt}

In temporal type theory, one models various types of behavior using data structures called \emph{sheaves}. 
A sheaf $B$ assigns a set $B(a,b)$ to each open interval $(a,b)\ss\rr$, understood as the set of \emph{possible behaviors} that can occur over the time-window $(a,b)$.
For any subinterval $(a',b')\ss(a,b)$, the sheaf also assigns a \emph{restriction map} $\rest{-}{a'}{b'}\colon B(a,b)\to B(a',b')$, which  crops the behavior $\beta\in B(a,b)$ to the subwindow, returning some $\rest{\beta}{a'}{b'}\in B(a',b')$.

Every aspect of the vehicle motion and the perception system considered in this paper has a representation in this common language of behavior types.
For example, at any moment of time, every variable $v\in V$ of type $\Type(v)$ has an associated a behavior type
\[
  \mathrm{Var}_v(a,b)\coloneqq\{x\colon(a,b)\to \Type(v)\tn{, continuous}\}
\]
where the restriction map $\rest{-}{a'}{b'}\colon \mathrm{Var}_v(a,b)\to\mathrm{Var}_v(a',b')$ is given by composition.
That is, given $x\in\mathrm{Var}_v(a,b)$, define $\rest{x}{a'}{b'}$ to be the composite $(a',b')\ss(a,b)\xrightarrow{x} \Type(v)$.
Below we will discuss behavior types for diagnosing faults in the perception pipeline as data arrives over intervals of time.
To do so, we first construct the behavior type $\mathrm{Grph}$ of all temporal diagnostic graphs.
A $(a,b)$-behavior, \ie element of $\mathrm{Grph}(a,b)$, consists of a graph $G=(U,E)$ together with a function $U\to(a,b)$, meaning that each node is assigned a moment in that interval, indicating the moment at which it occurs.
For any $(a',b')\in(a,b)$, the restriction map sends $G$ to $\rest{G}{a'}{b'}$, the maximal subgraph on the nodes that exist within the subinterval $(a',b')$; in other words it includes all the edges from $G$ between those nodes.

All possible behavior types, as sheaves, form a \emph{topos} $\cat{B}$.
A topos is a well-studied algebraic system that has many useful formal properties~\cite{MacLane12book-toposTheory,Johnstone02book-toposTheory}.
For example behavior types $A,B$ can be added (a behavior of $A+B$ is a behavior of $A$ or a behavior of $B$), multiplied (a behavior of $A\times B$ consists of a pair $(a,b)$ of behaviors), exponentiated, etc.
Moreover every topos comes equipped with its own \emph{internal language}, which includes a full-fledged constructive logic.
In the internal language of $\cat{B}$, the booleans are replaced with a notion of temporal truth values, a behavior type $\prop$ of all propositions.
It comes equipped with all the constants, operations, and quantifiers, $\true,\false,\wedge,\vee,\imp,\neg,\forall,\exists$.
The internal language can be used to provide very short and intuitive descriptions of complex objects.
One writes the description set-theoreatically, and then the topos machinery turns that description into one that varies in time.
For example, if one writes the standard definition of vector space in the internal language, the result will be a vector space changing in time. 

\myParagraph{Example}
We can use the internal language to define the sheaf of \emph{upper-bounded naturals}, denoted $\ol\nn$.
In the internal language it is defined by $\ol\nn\coloneqq\{u:\nn\to\prop\mid\forall n:\nn, u(n)\imp u(n+1)\}$, meaning ``if $n$ is an upper bound, then so is $n+1$''.
This translates to saying that a behavior $u\in\ol{\nn}(a,b)$ assigns a natural number to each open subinterval of $(a,b)$, and smaller subintervals can be assigned smaller natural numbers.

\subsection{Temporal Diagnostic Graphs as Behavior Types}\label{sec:tdgInterpretation}

Temporal diagnostic graphs such as the one in \cref{fig:dstacks} have a special property: any clipping to a subinterval has a lower diagnosability number.
To be more precise, let $D=(U,E)$ be a graph and for any $a\leq b\in\nn$, let $D_{[a,b]}=(U_{[a,b]},E_{[a,b]})$ denote the graph with vertices $U_{[a,b]}=\{(\tau,i)\in\nn\times U\mid a\leq \tau \leq b\}$ and edges
\begin{equation*}
  E_{[a,b]}=\big\{\left((\tau,i),(\tau,j)\right)\mid (i,j)\in E\big\}\;\cup\;\big\{((\tau,i),(\tau+1,i))\mid i\in U, a\leq \tau \leq b-1\big\}
\end{equation*}
Thus, the graph $U_{[a,b]}$ consists of $b-a+1$ copies of $D$, as well as edges connecting corresponding nodes across neighboring copies.
For example $D_{[a,a]}\cong D$, whereas we can imagine $D_{[a,a+1]}$ as two panes, each a copy of $D$, and an edge from each node in the first pane to the corresponding node in the second pane.  

Now for any frequency $f$ and graph $D=(U,E)$, consider a temporal graph $\ol{D}$ with $\ol{D}(a,b)=D_{\big[\ceil{a/f},\floor{b/f}\big]}$,%
\footnote{Note that the graph in \cref{fig:dstacks} has strictly more edges than $D_{[0,2]}$, so we can still conclude the result from~\cref{cor:more_edges}.}
 where each vertex $(i,u)$ is sent to the time $\ceil{a/f}+f\cdot i$. It consists of multiple panes, spaced $f$-apart, each isomorphic to $D$.
We call temporal graphs constructed as $\ol{D}$ was above \emph{time-regular graphs} and denote the subsheaf of time-regular graphs by $\mathrm{Grph_{tReg}}\ss\mathrm{Grph}$.

It follows from \cref{prop:subgraph_diag} that if one clips a time-regular graph $\ol{D}(a,b)$ to a smaller interval $(a',b')$, the result will have a smaller diagnosability number. 
Thus we obtain a map of sheaves $t\colon\mathrm{Grph_{tReg}}\to\ol{\nn}$, sending each time-regular graph $D$ to its diagnosability number $t(D)$.
Note that this contains a large amount of information: $D$ is not just one graph but a graph $D(a,b)$ for every duration of time, and similarly $t(D)$ consists of a natural number $t(D)(a,b)\in\nn$ for every interval $(a,b)$.

Using the language of temporal type theory we can also model fault detection over arbitrary intervals.
Suppose actual data is coursing through the perception pipeline.
For every interval $(a,b)$ of time, we have a diagnosability graph $D(a,b)$, and every node in it has published values for each of its variables.
This induces a function $\mathtt{Fault\_free}\colon U\to\prop$, that sends every \module $i\in U$ to the union of all subintervals throughout which it is fault-free.
The perception system is trying to infer this function, but on short intervals does not have enough data to determine it.
Thus it returns $\mathtt{Not\_known\_faulty}\colon U\to\prop$ with $\mathtt{Fault\_free}\imp \mathtt{Not\_known\_faulty}$.
For short intervals, when the $t$-diagnosability number is low, this will often be a strict inclusion: some faulty nodes may fail to be diagnosed.
However by auditing longer intervals, the decision layer can often \emph{ex post facto} diagnose the faulty nodes and obtain an equality.

\section{Conclusion}\label{sec:conclusions}

We presented a novel methodology for fault diagnosis in perception systems.
Towards this goal, we drew connections with the literature on diagnosability for multiprocessors systems 
and generalized it to account for heterogeneous \modules interacting over time, as the ones arising in perception systems.
In particular, we showed that considering the temporal dimension, \ie assessing the consistency of the system behavior over time, 
 has the potential to enhance diagnosability, while still enabling the use of existing tools for fault detection. 
Finally, 
we showed that the proposed monitoring approach can be elegantly described with the language of
 \emph{topos theory} which allows the formulation of diagnosability and fault detection over arbitrary time intervals.
 While the approach is fairly general and applicable to a variety of perception systems, we 
 considered a \emph{localization} system as a case study.


We believe that the proposed notion of \emph{temporal diagnostic graphs} can complement the literature and enhance the current 
practice, contributing to the goal of achieving safety and trustworthiness in autonomous vehicle applications.
For instance, a system designer can use the tools proposed in this paper to assess the diagnosability before deploying the vehicle on public roads, or design the system in such a way that its fault diagnosability is maximized.
While deployed, the proposed framework allows the vehicle to have a greater awareness of its operational envelope and 
enables real-time perception monitoring via runtime diagnosability.
Finally, for regulators, system-level guarantees provide a more solid and rigorous ground for certification, and,
 in the unfortunate case of an accident, the proposed approach increases accountability by providing
  formal evidence about the root cause of the  failures. 

This work opens a number of avenues for future work. First, we plan to provide more examples of perception systems 
that can be modeled with (and can benefit from)  the proposed monitoring approach. Second, we plan to build on the results presented in this 
paper to model an entire AV software architecture using the language of temporal type theory.



\printbibliography
\end{document}